\documentclass[letterpaper, 10 pt, conference]{ieeeconf}  

\IEEEoverridecommandlockouts                              

\usepackage{graphics} 
\usepackage{epsfig} 
\usepackage{times} 

\usepackage[backend=bibtex,style=ieee,mincitenames=1,maxcitenames=1,natbib=true]{biblatex} 
\usepackage[utf8]{inputenc}
\usepackage{hyperref}
\usepackage{url}
\usepackage{color}
\usepackage{amssymb}
\usepackage{amsmath}
\usepackage[capitalise,nameinlink,noabbrev]{cleveref}
\usepackage{breakcites}
\usepackage{bm}
\usepackage{graphicx}
\usepackage{subcaption}
\usepackage{xspace}
\usepackage{algorithm}
\usepackage{algorithmic}
\usepackage{array}
\usepackage[justification=centering]{caption}
\graphicspath{ {figures/} }

\newtheorem{theorem}{Theorem}

\newtheorem{definition}{Definition}

\definecolor{mypink1}{rgb}{0.858, 0.188, 0.478}
\definecolor{ao(english)}{rgb}{0.0, 0.5, 0.0}
\definecolor{ballblue}{rgb}{0.13, 0.67, 0.8}
\definecolor{blueviolet}{rgb}{0.54, 0.17, 0.89}
\definecolor{brickred}{rgb}{0.8, 0.25, 0.33}
\definecolor{burntorange}{rgb}{0.8, 0.33, 0.0}

\newcommand{\ouracronym}[0]{IBIT\xspace}

\title{\LARGE \bf
Intervention Design for Effective Sim2Real Transfer
}

\author{Melissa Mozifian$^{1}$, Amy Zhang$^{12}$,  Joelle Pineau$^{12}$, and David Meger$^{1}$ 
\thanks{$^{1}$ Montreal Institute of Learning Algorithms (MILA) and McGill University. $^{2}$ Facebook AI Research.
Correspondence to: Melissa Mozifian {\tt\small melissa.mozifian@mail.mcgill.ca}}%
}

\begin{document}

\maketitle
\thispagestyle{empty}
\pagestyle{empty}

\begin{abstract}
The goal of this work is to address the recent success of domain randomization and data augmentation for the sim2real setting. We explain this success through the lens of causal inference, positioning domain randomization and data augmentation as interventions on the environment which encourage invariance to irrelevant features.
Such interventions include visual perturbations that have no effect on reward and dynamics. This encourages the learning algorithm to be robust to these types of variations and learn to attend to the true causal mechanisms for solving the task. This connection leads to two key findings: (1) perturbations to the environment do not have to be realistic, but merely show variation along dimensions that also vary in the real world, and (2) use of an explicit invariance-inducing objective improves generalization in sim2sim and sim2real transfer settings over just data augmentation or domain randomization alone.
We demonstrate the capability of our method by performing zero-shot transfer of a robot arm reach task on a 7DoF Jaco arm learning from pixel observations.
\end{abstract}

\section{INTRODUCTION}
The use of simulation for robot training has many advantages, in terms of speed and cost.  But simulation-based training can be brittle, and lead to poor results when tranferred to the physical robot.  This can be alleviated substantially by sampling the domain randomization parameters to \emph{sufficiently cover} a range of parameters. Selecting which aspects of the simulation to vary is often left to human design. However, negative transfer (poor performance in the real world) can arise perturbing irrelevant factors, or even varying the right factors by too much or by too little. 

This paper provides a deeper understanding of what types of interventions are effective for good generalization.
We introduce the use of causal inference (CI) as both an analytical method to determine what randomizations (interventions in the CI terminology) are useful, and to provide powerful learning tools that automatically pay attention to relevant variations, while ignoring spurious correlations among the random factors.
Our approach further shows that the types of data augmentation and domain randomization used do not have to be physically realistic, but must be used on components of the environment that truly vary without affecting the dynamics or reward function. This enables us to make as few assumptions as possible about the features of the real world, as long as training in simulation captures the fundamental visual features required in order to solve the task at hand.

Recent connections between state abstractions, specifically bisimulation \cite{ferns2011contbisim} [\cref{sec:bisim}], and causal inference~\cite{zhang2020invariant}, help determine the difference between modifications to the environment that permit \textit{positive}, as opposed to negative, transfer.
If the true underlying causal model of the generated data is known, this knowledge can be leveraged to design data augmentation techniques appropriate to the task~\cite{peters2016icp}. 
Recent results in the model-free regime to solve robotics tasks have shown significant gains through the use of data augmentation~\cite{kostrikov2020drq,laskin2020rad} in the single-task setting. However, current deep RL models trained with observational data and the principle of empirical risk minimization~\cite{NIPS1991_506} fail to generalize to unseen domains. Domain randomization and data augmentation methods do not explicitly take spurious correlations into account. In more practical applications such as robotics, data augmentation techniques are merely heuristics to enable learning invariant features.

In this work, we show how domain randomization combined with an invariance objective can weaken the spurious correlations in observed domains, enabling RL agents to generalize to unseen domains.
Furthermore, we show what assumptions are needed to learn a minimal, causal representation of an environment and connect these findings back to explain the success of domain randomization and data augmentation. To bridge the gap between theory and practice, we explain why this method works well for sim2real transfer and show that the incorporation of an additional invariance objective further improves generalization performance.
We introduce a practical method, \textbf{I}ntervention-\textbf{b}ased \textbf{I}nvariant \textbf{T}ransfer learning (\ouracronym), that uses both data augmentation and domain randomization in a multi-task setting, additionally enforcing invariance across domains to obtain sample-efficient transfer in sim2sim~\cite{yu2019metaworld} and sim2real settings. This real-world transfer experiment for a robot manipulation task is, to our knowledge, the first demonstration of generalization capability of current state-of-the-art pixel-based control methods in deep RL. For sample videos and link to the code base see \mbox{\url{https://sites.google.com/view/ibit}}.

\begin{figure}[bh]
    \centering
    \includegraphics[width=.85\linewidth]{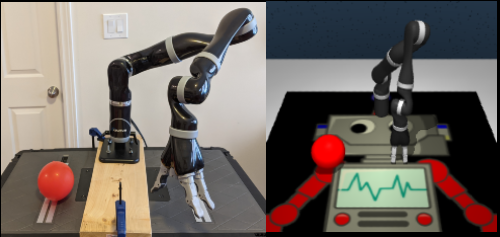}
    \caption{\small Sim2real Jaco robot arm setup.}
    \vspace{-10pt}
    \label{fig:robot_teaser}
\end{figure}


\section{TECHNICAL BACKGROUND}
\label{sec:bg}
In this section, we introduce assumptions about the environment, defined as latent structure in Markov Decision Processes (MDPs), the transfer learning setting, and concepts from causal inference and state abstraction theory that are leveraged in this work.

\subsection{Structured MDPs}
\label{sec:bg_mdp} 

We assume the domains can be described as \textit{Markov Decision Processes}, defined by tuple $\langle \mathcal{S}, \mathcal{A}, P, R \rangle$ with state space $\mathcal{S}$ and action space $\mathcal{A}$. $P$ denotes the latent transition distribution $P(s'|s,a)$ for $s,s'\in\mathcal{S}, a\in\mathcal{A}$, and $R(s,a)$ the reward function. 
To address visual differences, we assume the true state space is latent, and the agent instead receives information  from an additional observation space, $\mathcal{X}$.  We additionally define an emission mapping $q$, which denotes the mapping from state to observation. Data augmentation and domain randomization can be framed as interventions on this mapping $q$. These changes across domains can be denoted by a domain-specific mapping $q_d$. 


\subsection{Domain  Randomization and Data Augmentation}
Domain randomization is an approach where one tries to find a representation that generalizes across different environments, called \textit{domains}.
\citet{tobin2017domain} introduced domain randomization which randomizes the rendering configuration in the simulator and with enough variability in the simulator, the real world could appear, as just another variation. 
In this setting, we call an environment that we have full access to a source domain and the environment that we want to transfer the model to, a target domain. Training proceeds in the source domain and a set of $N$ randomization parameters are used to expose the policy to a variety of environments which help it to generalize. In this work we only consider domain randomizations that affect the mapping $q$.

Data augmentation is a very similar framework, but limited to changes to the observation that can be made after the mapping $q$. Data augmentation has been found to improve generalization performance in supervised learning settings~\cite{volpi2018da} and RL settings where access to the simulator engine is not assumed~\cite{kostrikov2020drq,srinivas2020curl,laskin2020rad}.

\subsection{Causal Discovery  \& Invariant Prediction}
Causal discovery aims to find causal relations from data. We assume a \textit{Structural  Causal Model}~\cite{pearl2009do}, or a latent DAG (Directed  Acyclic Graph) structure underlying the data generation, which  differs from conditional observational dependence between variables.  
Such a model will reflect on the underlying `true' causal structure of domain generalization problem.
However, causal structure cannot be found from purely observational data~\cite{peters2016icp}. \citet{peters2016icp} details the types of \textit{interventions} necessary to find the correct causal graph structure. Under these interventions,  
invariance can be used as a proxy for causal relationships, as those causal relationships will stay constant under interventions. 
Based on this insight, \citet{peters2016icp} introduce an algorithm, Invariant Causal  Prediction (ICP), to discover causal graph structure. However, it  relies on access to the variables and is super-exponential  in the number of variables.
More recently, \citet{arjovsky2019irm,krueger2020rex} propose gradient-based methods for learning invariant representations. 
A form for learning invariant representations in MDPs was introduced by \cite{zhang2020invariant}. We extend this work to the sim2real setting, and provide guidelines as to the types of interventions necessary to achieve the necessary generalization. The types of interventions relevant to the sim2real setting can be found in \cref{sec:interventions}. 

\subsection{State Abstractions}
\label{sec:bisim} 
\textbf{Bisimulation} is a form of state abstraction that groups states $s_i$ and $s_j$ that are ``behaviorally equivalent''~\cite{li2006stateabs}. For any action sequence $a_{0:\infty}$, the probabilistic sequence of rewards from $s_i$ and $s_j$ are identical.
A more compact definition has a recursive form: two states are bisimilar if they share both the same immediate reward and equivalent distributions over the next bisimilar states~\cite{larsen1989bisim,Givan2003EquivalenceNA}. 

\begin{definition}[Bisimulation Relations~\cite{Givan2003EquivalenceNA}]
Given an MDP $\mathcal{M}$, an equivalence relation $B$ between states is a bisimulation relation if, for all states $s_i,s_j\in\mathcal{S}$ that are equivalent under $B$ (denoted $s_i\equiv_Bs_j$) the following conditions hold:
\begin{alignat}{2}
    \mathcal{R}(s_i,a)&\;=\;\mathcal{R}(s_j,a) 
    &&\quad \forall a\in\mathcal{A}, \label{eq:bisim-discrete-r} \\
    \mathcal{P}(G|s_i,a)&\;=\;\mathcal{P}(G|s_j,a) 
    &&\quad \forall a\in\mathcal{A}, \quad \forall G\in\mathcal{S}_B, \label{eq:bisim-discrete-p}
\end{alignat}
where $\mathcal{S}_B$ is the partition of $\mathcal{S}$ under the relation $B$ (the set of all groups $G$ of equivalent states), and $\mathcal{P}(G|s,a)=\sum_{s'\in G}\mathcal{P}(s'|s,a).$
\end{definition}

Exact partitioning with bisimulation relations is generally impractical in continuous state spaces, as the relation is highly sensitive to infinitesimal changes in the reward function or dynamics. For this reason, \textbf{Bisimulation Metrics}~\cite{ferns2011contbisim,ferns2014bisim_metrics} instead defines a pseudometric space $(\mathcal{S}, d)$, where a distance function $d:\mathcal{S}\times\mathcal{S}\mapsto\mathbb{R}_{\geq 0}$ measures the ``behavioral similarity'' between two states\footnote{Note that $d$ is a pseudometric, meaning the distance between two different states can be zero, corresponding to behavioral equivalence.}.
The bisimulation metric is the reward difference added to the Wasserstein distance between transition distributions:
\begin{definition}[Bisimulation Metric]
\label{def:bisim_metric}
From Theorem 2.6 in \cite{ferns2011contbisim} with $c\in[0,1)$:
\begin{align}
d(s_i,s_j) 
\;&=\; \max_{a\in \mathcal{A}}\;
(1-c)\cdot |\mathcal{R}_{s_i}^a - \mathcal{R}_{s_j}^a| + c\cdot W_1(\mathcal{P}_{s_i}^a,\mathcal{P}_{s_j}^a;d).
\label{eq:bisim_metric}
\end{align}
\end{definition}

\citet{zhang2020invariant} first drew connections between bisimulation and invariant prediction. Here we show how to use bisimulation metrics to find the coarsest bisimulation partition as a way to enforce invariance across interventions.




\begin{figure}[t!]
 \centering
 \begin{subfigure}[t]{0.78\linewidth}
    \includegraphics[width=1\linewidth,trim=0 0 610 0, clip]{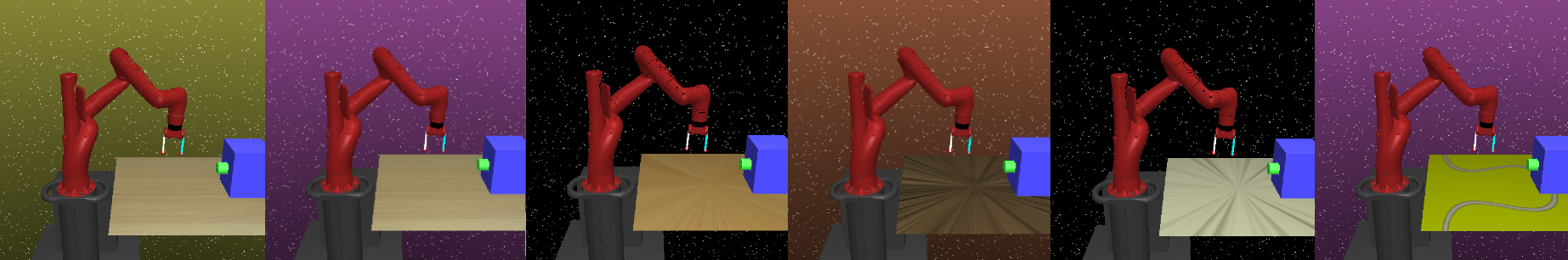}
    \centering
    \label{fig:buttonpress_train}
    \vspace{-10pt}
  \end{subfigure}
  \begin{subfigure}[t]{0.2\linewidth}
    \includegraphics[width=.97\linewidth]{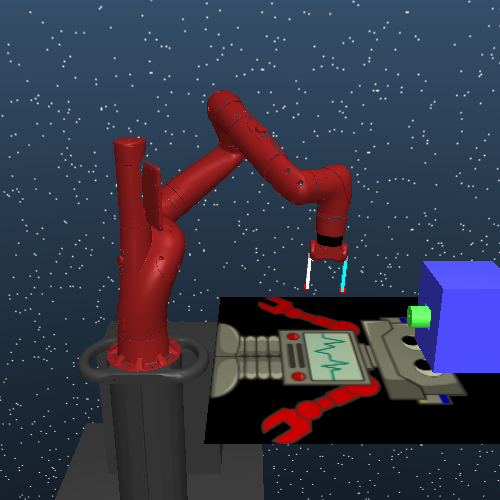}
    \centering
    \label{fig:buttonpress_eval}
        \vspace{-10pt}
  \end{subfigure}
   \begin{subfigure}[t]{0.78\linewidth}
    \includegraphics[width=1\linewidth,trim=0 0 610 0, clip]{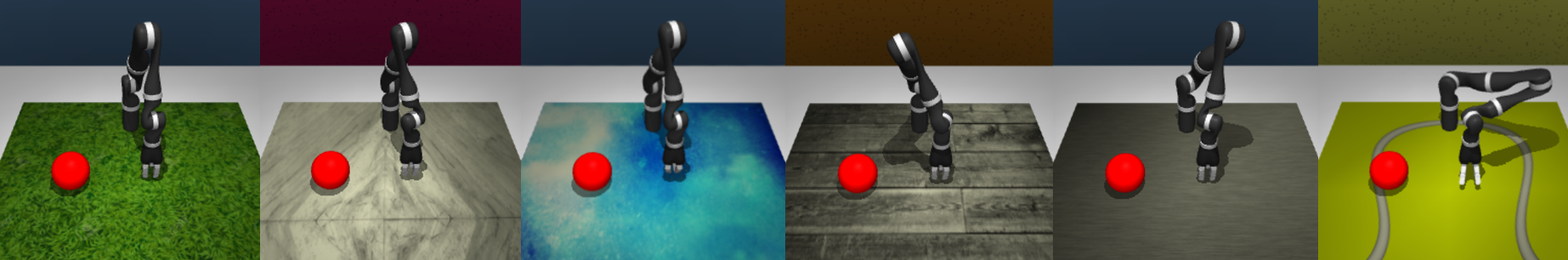}
    \centering
    \caption{\small Training envs with randomized visual rendering.}
    \label{fig:jaco_train}
  \end{subfigure}
  \begin{subfigure}[t]{0.2\linewidth}
    \includegraphics[width=.94\linewidth]{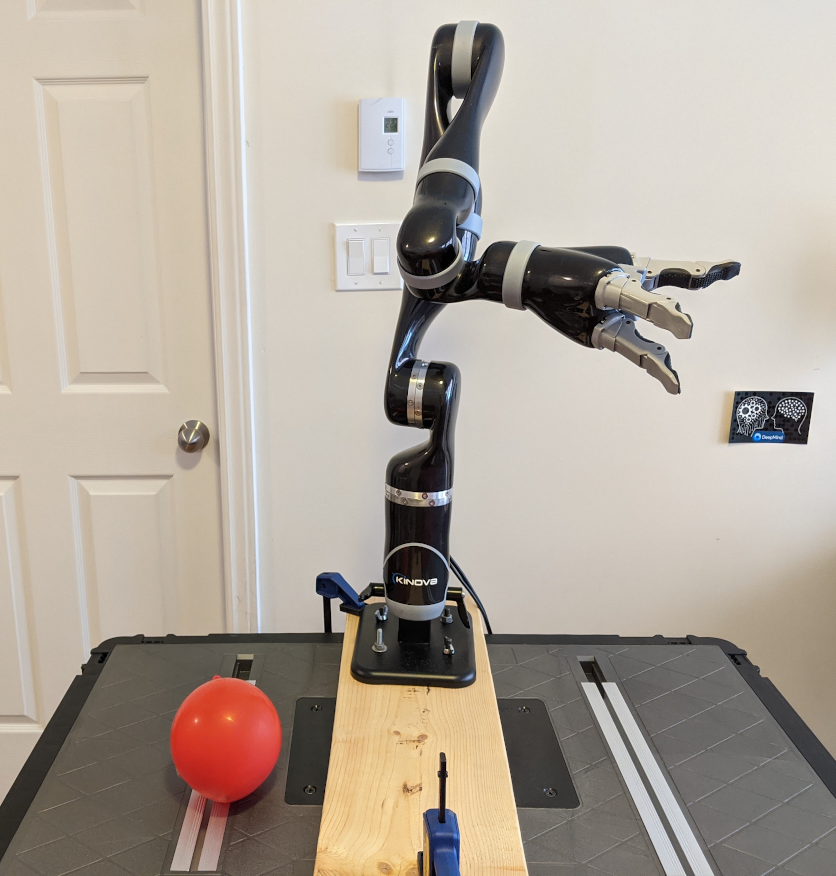}
    \centering
    \caption{\small Test time.}
    \label{fig:jaco_eval}
  \end{subfigure}
  \caption{\small Training and evaluation settings. At training time, the agent is exposed to variable changes such as background and table colour while relevant visual features such as the arm itself and target are kept constant. At test time, the agent is evaluated on \textit{new unseen} environments.}
  \vspace{-20pt}
  \label{fig:setup}
\end{figure}

\section{Problem Setting} 
\label{sec:problem_setting}
We assume access to a simulator where we can modify the rendering function. These modifications give rise to different environments, which we treat as a multi-task setting. This requires access to an environment id that allows the agent to determine when it has changed environments, but this assumption can also be relaxed, as discussed in \cref{sec:methods}. 

\subsection{Forms of Interventions}
\label{sec:interventions}
We posit two forms of possible interventions in the simulation setting that are relevant for sim2real. 
The first is \textit{post-rendering interventions}, one form of which is data augmentation, or an intervention staged on the observation $x$ after it is rendered from the latent state  $s$. This has been shown to be useful in RL settings~\cite{laskin2020rad,kostrikov2020drq}, and is the most versatile in that it doesn't require access to the simulation engine or emission function. Examples of this type of intervention include crops, flips, or rotations. Specifically, it refers to any reasonable transformations that can be done to the observation, after it is rendered. Such transformations should be selected in a way that does not hinder the RL reward function. For instance, consider a reacher task where the goal of the agent is to reach a red goal. If the transformation alters the goal colour in any way, this can cause inconsistency to the observation and the reward function.

The second type of intervention we consider is \textit{rendering interventions}, that is, modifying the environment or simulator to intervene directly on the emission function $q_d:\mathcal{S}\mapsto \mathcal{X}$. This would also modify only the observation and not the underlying dynamics or reward function of the MDP, but generally requires access to the rendering engine in practice. Examples  of this type of intervention include different camera angles or colour or texture changes of the background or objects in the scene. Some forms of this don't require access to the rendering engine with inductive bias and image processing techniques, like interventions described in \cite{azhang2018natrl} of injecting natural video into backgrounds of simulators. 



\subsection{Multi-Task Setting}
We treat interventions as separate tasks, and use the goal of learning an invariant representation across tasks to generalize to unseen, new tasks in a zero-shot manner.
We assume access to $N$ different domains at training time. For each domain, an intervention tag is assigned which indicates the type of intervention, i.e. augmentation to the domain.
In order to test the ability of the model to generalize, we evaluate in a previously unseen test domain $d = N + 1$.




\section{CONNECTIONS TO CAUSAL INFERENCE}

In this section we present results on the range of interventions needed to guarantee generalization, as well as the restrictions on those interventions to guarantee they will always help transfer rather than harm it.
\begin{theorem}[Identifiability]
\label{thm:identifiability}
We assume underlying variables $\{X_1,X_2,...,X_n\}$ make up the state space $\mathcal{S}$ of MDP $\mathcal{M}$, with a linear Gaussian relationship with the reward. If at training time we have access to an intervention on each variable $X_i,i\in\{1,...,n\}$, the causal parents of the reward, $PA(R)$, are identifiable.
\end{theorem}
The proof follows directly from Theorem 3 in \citet{peters2016icp}. 
\cref{thm:identifiability} gives us the requirements necessary to obtain desired generalization in sim2real settings. We can define and intervene on specific variables we know can vary across simulation and real settings, e.g. table top color and texture, background, and/or mass of objects. Further, these interventions don't need to be ``realistic". We can use a texture never seen at evaluation time, as shown in \cref{fig:setup}, or an extreme friction coefficient never seen in the family of real tasks. Rather, what is important is the presence of the variation itself across training tasks.

\begin{definition}[Valid interventions]
\label{def:valid}
An intervention is \textit{valid} if it positively affects generalization performance.
\end{definition}
\cref{def:valid} defines the type of intervention we want to design. The below result gives guidance on how to find valid interventions.
\begin{theorem}[Bisimulation validity]
\label{thm:bisim_valid}
All interventions that result in an MDP that is bisimilar to the original are valid.
\end{theorem}
\begin{proof}
From the causal inference literature, an intervention is always valid unless it is intervening on the target variable~\cite{peters2016icp}, which in the RL setting pertains to the reward and dynamics of future reward. We use Theorem 1 from \citet{zhang2020invariant}, which proves that the causal ancestors of the reward are the causal feature set of an MDP, and also correspond to the coarsest bisimulation partition. Therefore, any intervention that generates an MDP that is bisimilar to the original is valid.
\end{proof}
The benefit of domain randomization techniques where the simulator is directly modified means it is easier to design valid interventions that do not violate \cref{def:valid}, which was an issue discussed by \citet{kostrikov2020drq} when performing post-rendering interventions.

\section{METHODS FOR LEARNING INVARIANCE OVER INTERVENTIONS}
\label{sec:methods}

\begin{figure}
    \centering
    \includegraphics[width=\linewidth]{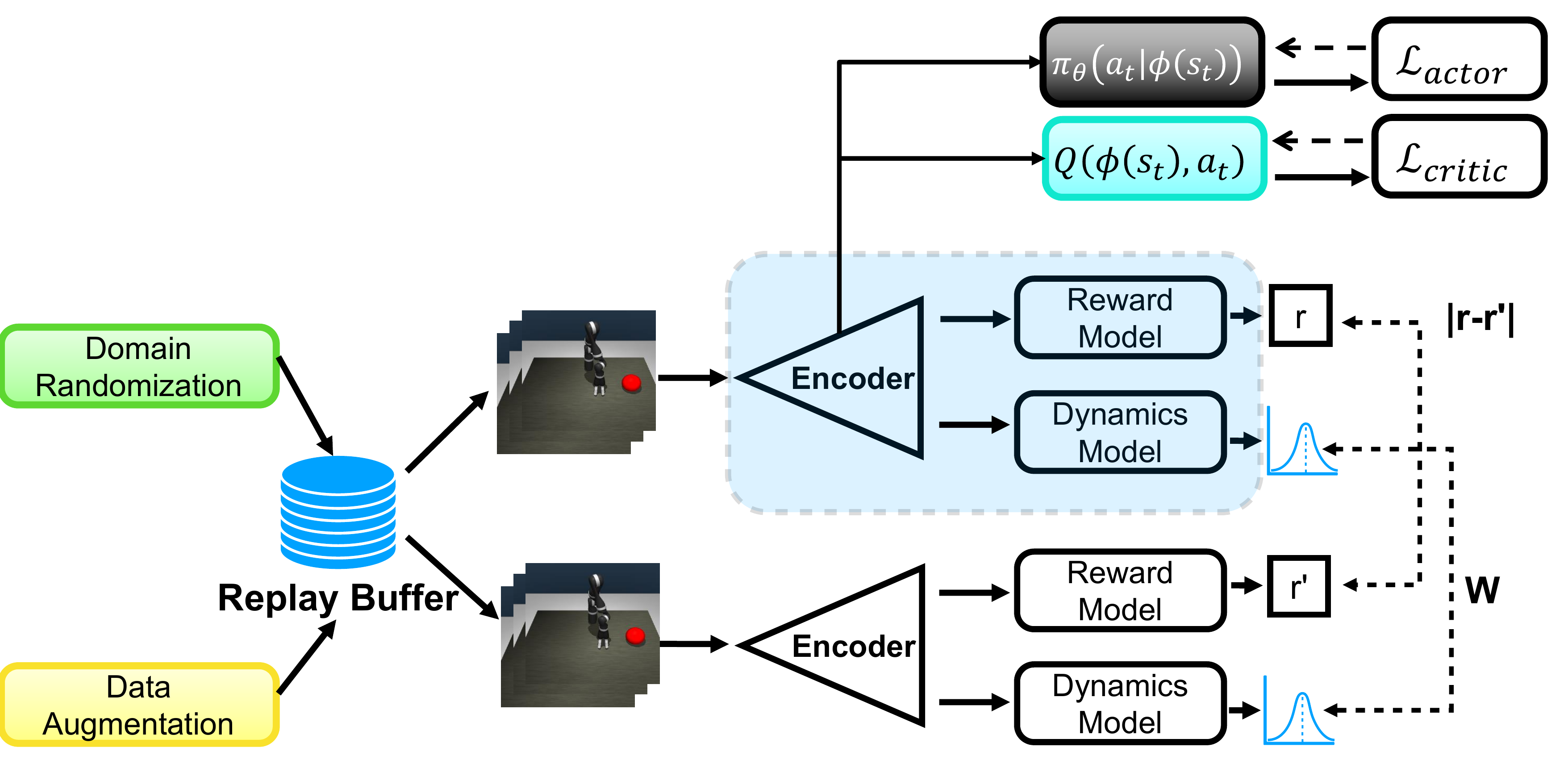}
    \caption{\small Flow diagram of our method, \ouracronym. Shaded in blue is the main model architecture, it is reused for both states, like a Siamese network.}
    \label{fig:architecture}
\end{figure}

Our goal is to leverage interventions, modeled as different domains in a multi-task setting, to learn invariant representations that generalize across simulation domains, and from simulation domains to real ones. We explore two approaches to invariance: 1) using \textit{bisimulation metrics} to learn the coarsest bisimulation partition, the fewest clusters of behaviorally equivalent states where similarity is measured by current and future reward, and 2) \textit{risk extrapolation}, a way to encourage robustness to affine combinations of tasks.

\setlength\textfloatsep{2pt}
\begin{algorithm}[t]
   \caption{\small \ouracronym and \ouracronym-REx}
   \label{alg:main_alg}
\begin{algorithmic}
    \FOR {Time $t = 0$ to $\infty $}
    \STATE Sample a batch of visually randomized environments
    \FOR {env in envs\_batch} 
    \STATE \textcolor{ao(english)}{Apply data augmentation to pixel frames}
    \STATE Encode observations $z_t = \phi(s_t)$
    \STATE Sample action from policy, $a_t \sim \pi(.|s_t)$
    \STATE Step in environment, $s_t' \sim p(.|s_t, a_t)$
    \STATE $\mathcal{D} \leftarrow \mathcal{D} \cup (s_t, a_t, r_t, s_t', \text{env\_tag})$
    \STATE     UpdateCritic($\mathcal{D}$) $\triangleright$ {\color{ao(english)}{Apply Data Augmentation Reg} }
    \STATE     UpdateActor($\mathcal{D}$)
    \STATE {Sample batch $\mathcal{B}_i \sim \mathcal{D}$}
    \STATE \textcolor{burntorange}{Permute batch randomly: $\mathcal{B}_j = permute(\mathcal{B}_i)$}
    \STATE {Train policy: $\mathbb{E}_{\mathcal{B}_i}[J(\pi)] $}
    \STATE \textcolor{burntorange}{Train encoder: $\mathbb{E}_{\mathcal{B}_i, \mathcal{B}_j}[J(\phi)] $}
    \STATE \textcolor{burntorange}{Train dynamics: $J(\hat{\mathcal{P}}, \phi) = (\hat{\mathcal{P}} (\phi(s_t),a_t) - \bar{z}_{t+1}  )^2 $}
    \small{\STATE \textcolor{burntorange}{Train reward: $ J(\hat{\mathcal{R}}, \hat{\mathcal{P}} , \phi ) = (\hat{\mathcal{R}}(\hat{\mathcal{P}} (\phi(s_t),a_t))-r_{t+1})^2$}}
    \IF{\textcolor{ballblue}{REx}}
    \STATE \textcolor{ballblue}{Apply penalty term $\mathcal{R}_\textrm{V-REx}$ from Eq. \ref{eqn:V-REx}}
    \ENDIF
    \ENDFOR
    \ENDFOR
\end{algorithmic}
\end{algorithm}

\subsection{Leveraging Bisimulation Metrics for Invariance}
To train the encoder $\phi$ towards the desired relation $d(x_i, x_j) := ||\phi(x_i) - \phi(x_j)||_1$, we draw batches of observation pairs and minimise the mean square error between the bisimulation metric and $\ell_1$ distance in latent space~\cite{zhang2020bisim}: 
\vspace{-5pt}
\begin{align}
    J(\phi) 
    &\;=\; \Big(||z_i - z_j||_1
    \;-\; |\hat{R}(\bar{z}_i) - \hat{R}(\bar{z}_j)| \\
    &\;-\;  \gamma\cdot W_2\big(\hat{P}(\cdot|\bar{z}_i,\bar{\pi}(\bar{z}_i)),\, \hat{P}(\cdot|\bar{z}_j,\bar{\pi}(\bar{z}_j))\big)\Big)^2, \nonumber
    \label{eq:bisim_loss}
\end{align}
This bisimulation metrics objective for learning the coarsest bisimulation partition under interventions is our base method, \ouracronym. Note that this invariance-inducing objective does \textit{not} require access to environment ids.

\subsection{Leveraging Interventions}
We can also explicitly add an invariance regularization term to the reward and transition models using risk extrapolation~\cite{krueger2020rex} with the motivation of discovering stable features as means of generalization.
V-REx~\cite{krueger2020rex} enforces the risk $\mathcal{R}$ to be close across all tasks. For our setting, the risk can be aggregated over the latent transition model or reward model, or both, defined in \cref{eq:risk}:
\vspace{-5pt}
\begin{equation}
\label{eq:risk}
  \mathcal{R}_i(\phi) = MSE(R, \hat{R}) + MSE(P, \hat{P}).
\end{equation}
 \vspace{-5pt}
The V-REx objective is as follows,
\begin{align}
    \label{eqn:V-REx}
    \smash{\mathcal{R}_\textrm{V-REx}(\phi) \doteq \beta \mathrm{Var}(\{\mathcal{R}_1(\phi), ..., \mathcal{R}_m(\phi)\}) + \sum^m_{e=1} \mathcal{R}_e(\phi).}
\end{align}
Here $\beta \in [0, \infty)$ controls the balance between reducing average risk and enforcing equality of risks, with $\beta=0$ recovering ERM, and $\beta \rightarrow \infty$ leading V-REx to focus entirely on making the risks equal. The addition of this risk extrapolation objective to control for invariance we denote \ouracronym-REx. We note that this objective does require the multi-task setting with access to environment ids.

\cref{alg:main_alg} shows our method for learning invariances. \textit{Rendering interventions} are used to generate multiple environments, which we train over in a multi-task manner. Bisimulation metrics steps are in \textcolor{burntorange}{orange}, post-rendering intervention steps in \textcolor{ao(english)}{green}, and risk extrapolation in \textcolor{ballblue}{blue}. \cref{fig:architecture} (best viewed in colour) presents another view of our method, where we use a Siamese network to compare distances between randomly drawn states to train the encoder, then train soft actor-critic on the latent states.

\section{EXPERIMENTS \& RESULTS}
In this section, we evaluate \ouracronym on sim2sim and sim2real generalization tasks. Our experiments are designed to answer the following questions: 1) Can we get better generalization performance by staging a thorough set of interventions in simulation using data augmentation and domain randomization techniques? 2) Do ``unrealistic" modifications to the environment also help generalization to real settings? 3) Does leveraging invariances improve performance beyond just interventions? 

Our experiment setup is as follows: We use the open-source implementation of DrQ~\citep{kostrikov2020drq}, which implements Soft Actor-Critic with data augmentation including random crops and random shifts which were found to be  effective for locomotion-based RL tasks. We further modified the training environments to incorporate visual domain randomization of the background colour and top table texture and colour. We use DBC~\cite{zhang2020bisim} to find the coarsest bisimulation partition via bisimulation metrics, and the REx penalty term \cite{krueger2020rex}. In addition to the pixel observations, we also append low level states such as robot joint information, to the output of the observation encoder, to fully utilize all available information which would also be present at test time.
Our experiments show the effect of using this regularization term in the performance of the policy. All of our evaluations are performed on unseen environments, where the agent does not observe the selection of colour and textures at training time. Our experiments demonstrate that the exact choices of unseen textures or colours do not matter, as long as the important features for the task itself are consistently present both during training and test time. All other details such as background  should be ignored during training.


\subsection{Sawyer Manipulation Tasks}
We tested the following manipulation tasks with dense reward from the Metaworld Benchmark Suite \cite{yu2019metaworld} including \textit{Reacher\_v1}, \textit{Button\_Press\_v1}, \textit{Window\_Open\_v1} and \textit{Window\_Close\_v1}. We selected these environments in particular, as they offer challenging visual manipulation tasks.
Each of these tasks exhibit a single goal setting where the agent is expected to reach a target, or manipulate an object towards a goal configuration. 
Observations from example train and test environments can be found in \cref{fig:setup}. We see that the training environments are not ``realistic" in the sense that they do not attempt to be as close to the real setting as possible.
For baselines, we evaluate \ouracronym and \ouracronym-REx against methods that don't explicitly induce invariance but still utilize both rendering and post-rendering interventions (DrQ), and those that only utilize rendering interventions, i.e. just soft actor-critic in our multi-task setting (SAC).

\begin{figure*}
    \centering
    \includegraphics[width=0.7\linewidth]{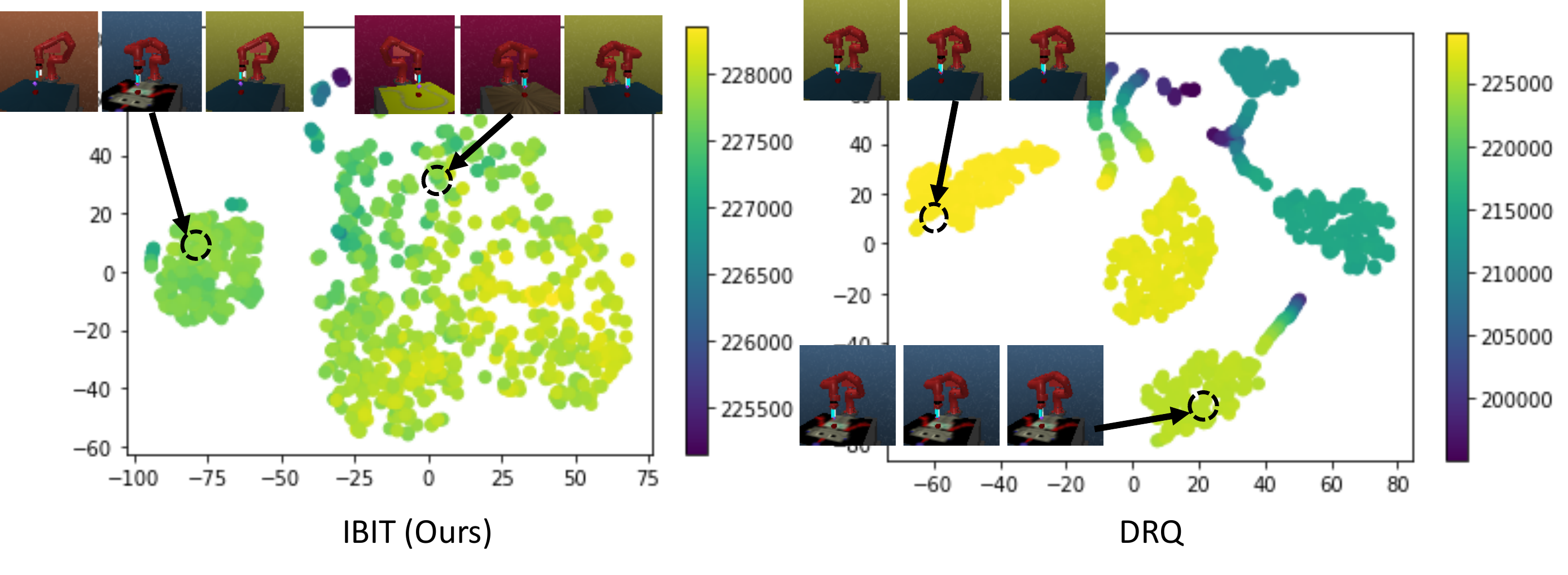}
    \vspace{-10pt}
\caption{\small t-SNE comparison of the latent representations learned by \ouracronym and DrQ. colours represent the predicted value by the critic. Our method successfully learns an invariant representation to domain randomization, while DrQ does not.}
\label{fig:tsne}
\end{figure*}

\begin{figure*}[h]
    \vspace{-10pt}
    \centering
    \includegraphics[width=.99\textwidth]{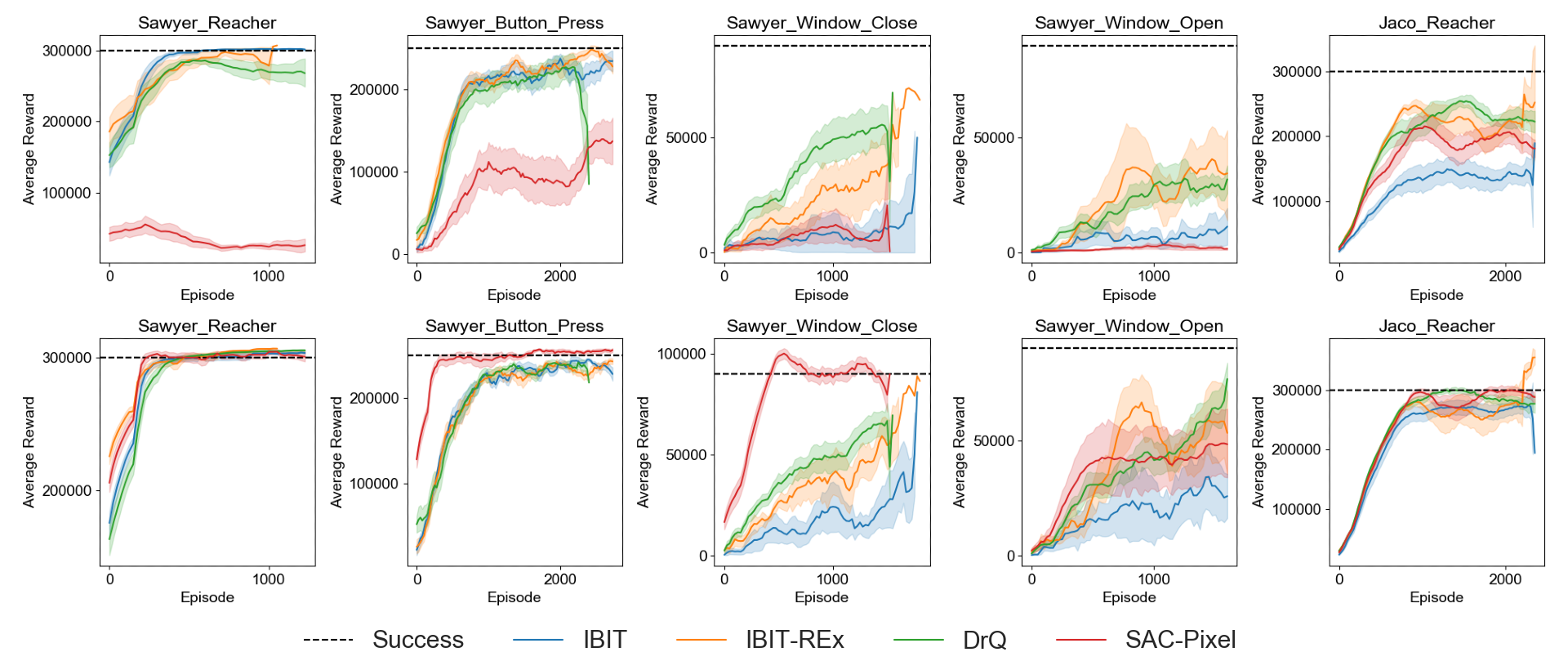}
    \vspace{-1pt}
    \caption{\small Comparison of methods on unseen environment (top) and seen environment (bottom).}
    \label{fig:unseen_eval}
    \vspace{-10pt}
\end{figure*}

\subsection{Jaco 7DoF Reacher}
We run both sim2sim and sim2real transfer for the Jaco arm reacher task. In contrast to the Sawyer tasks, where the agent is learning to control the end-effector gripper of the arm, to make the task more challenging, we allow the agent to control all the joints, which in turn, brings more degrees of freedom in terms of controlling the arm and it can be more realistic for certain manipulation tasks. Similar to the Sawyer experiments, we evaluate the sim2sim performance on a visually unseen environment.

For the main Sim2Real experiments, we train a 7DoF Jaco arm \cite{campeau2019kinova} in simulation. The simulated Jaco arm used within these environments is part of the DeepMind control suite \cite{Tunyasuvunakool2020}, with modifications made to visually randomize the scene.
For this task, the agent learns to controls joints velocities of the 7DoF robot arm. We assume dense reward during training and only sparse reward during testing on the real arm. Although this is a single goal reach task, the arm is reset to a random initial position at the beginning of each episode, making the task of generalization slightly harder as the agent needs to find a trajectory towards the target goal from any starting joint configuration. This implies that the arm starts in a new position which visually looks different. 

\subsection{Results}
\cref{fig:unseen_eval} shows results across all sim2sim tasks aggregated over $10$ seeds for both Sawyer and Jaco -- addressing questions 2) and 3) by confirming that invariance-seeking objectives do help generalization performance, and we can successfully generalize with unrealistic interventions. These results show performance on both seen (during training) and unseen environments. As expected, without the invariance objective, the policy overfits to the training environments and is unable to solve the task in the unseen case.
Our results show IBIT and IBIT-REx are both effective at generalizing, with their performance dependent on the task and the choices of hyperparameters per task. This is with the exception to the \textit{Sawyer\_window\_Open} task where all methods struggle to converge due to the difficulty of the task. We hypothesize that this task is visually more challenging compared to others and none of the methods are able to reliably generalize in the unseen environment case.
We also plot ablations of how well DBC as an invariance method generalizes to unseen environments under incomplete forms of interventions in \cref{fig:dr_da_ablation}, which answers question 1) that both types of interventions are helpful. Finally, \cref{fig:tsne} shows a t-SNE comparison of the latent representations learned in Sawyer \textit{Reacher\_v1} by IBIT and DrQ. We see that IBIT successfully learns invariances and plots equivalent states to the same area in latent space, even though they have large visual differences. DrQ learns a latent space that strongly correlates with estimated value, but does not learn to plot equivalent states under different interventions near each other.

The results of the sim2real arm experiments are demonstrated in \cref{tab:jacotrial}. We ran the best policy of each method $10$ times, and success was evaluated based on the arm successfully touching the target. The results demonstrate the robustness of each method in a real world setting where there are unseen variations of the visual representation that challenge the true generalization capability of each method.

\begin{figure}
    \centering
    \includegraphics[width=\linewidth]{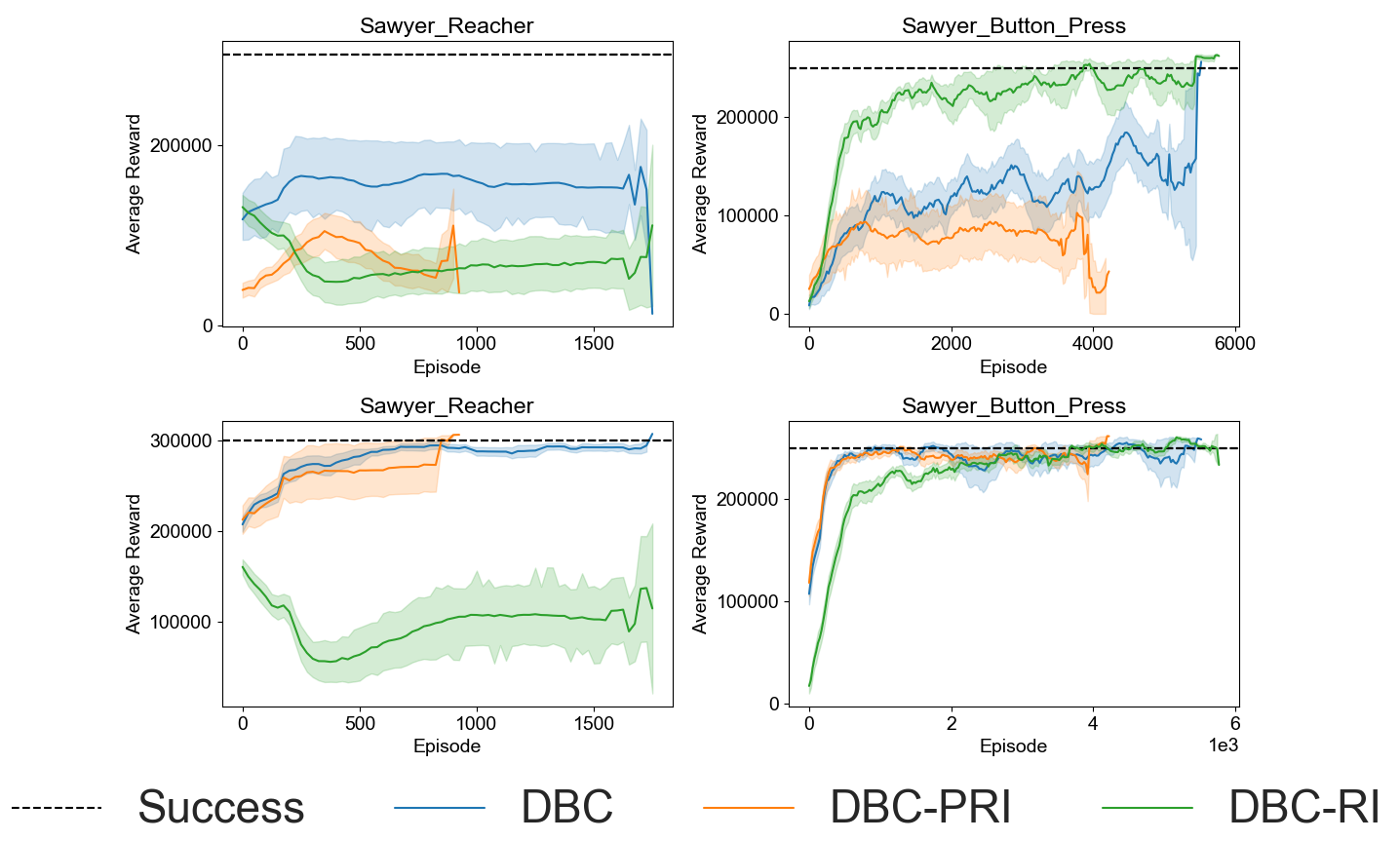}
    \vspace{-10pt}
    \caption{\small Comparison of methods with and without PRI (post-rendering interventions), and with and without RI (rendering interventions) on unseen (top) and seen environments (bottom).}
    \label{fig:dr_da_ablation}
\end{figure}

\newcolumntype{P}[1]{>{\centering\arraybackslash}p{#1}}
\begin{table}
\centering
\begin{tabular}{ |P{3cm}||P{3cm}|  }
 \hline
 \multicolumn{2}{|c|}{Jaco Trials} \\
 \hline
 Method & Success Rate \\
 \hline
 \ouracronym   & $\mathbf{8/10}$ \\
 \ouracronym-REx & $7/10$ \\
 DrQ & $5/10$ \\
 \hline
\end{tabular}
\caption{\label{tab:jacotrial} \small Real Jaco transfer success rate over $10$ trials.}
\end{table}

\section{RELATED WORK}
Prior work has been done in domain randomization for sim2real transfer, data augmentation for better generalization in RL settings, and invariance finding through the causal inference lens, which we detail below.

\citet{tobin2017domain} introduced \textbf{domain randomization} for \textbf{sim2real transfer}, which randomizes the rendering configuration in the simulator, and posited that with enough variability in the simulator, the real world could appear just as another variation. However, they did not specify what types of randomization are required to bridge the reality gap.
\citet{rusu2016sim} explore using the progressive network architecture to adapt a model that is pre-trained on simulated pixels.
\citet{rao2020rlcyclegan} proposes to automatically bridge the simulation-to-reality gap by employing generative models to translate simulated images into realistic ones. Such transformation is task-agnostic and the transformed images may not preserve all task-agnostic features. The authors propose a RL-scene consistency loss which ensures the image transformation is invariant with respect to the Q-values associated with the images. This enables learning task-aware transformations, but requires transferring simulated environments to match reality, which \ouracronym shows is not a necessary factor. 

\textbf{Data Augmentation} has gained recent popularity in self-supervised learning for improving generalization performance in supervised learning and reinforcement learning.
Several contrastive representation learning approaches \cite{srinivas2020curl, chen2020simple} utilize data augmentations and perform patch-wise or instance discrimination. The contrastive objective aims to maximize agreement between augmentations of the same image and minimize it between all other images in the dataset.
More recent work found that the contrastive objective is unnecessary, and even harmful, for representation learning in RL~\cite{laskin2020rad}.
\citet{kostrikov2020drq} propose augmentations and weighted Q-functions in conjunction with Soft Actor-Critic (SAC).
They utilize standard image transformations, specifically random shifts, to regularize the Q-function learned by the critic so that different shifts to the same input image have similar Q-function values.
Also inspired by the impact of data augmentation in computer vision \citet{laskin2020rad} propose RAD: Reinforcement Learning with Augmented Data, showing that data augmentations can enable simple RL algorithms to match and even outperform complex state-of-the-art methods such as \cite{lee2019stochastic,hafner2018learning,hafner2019dream} on the task of visual dynamics learning.
\citet{cobbe2018quantifying} and \citet{lee2020network} also show that simple data augmentation techniques such as cutout~\cite{cobbe2018quantifying} and random convolution~\cite{lee2020network} can be useful to improve generalization of agents on the OpenAI CoinRun and ProcGen benchmarks.

\citet{peters2016icp} first introduced \textbf{invariance as a proxy for causal mechanisms}, using hypothesis testing methods to find the set of variables that cause a target variable. Their method unfortunately relies on access to the true variables and is super-exponential in the number of variables, and therefore is not widely applicable in realistic settings. \citet{arjovsky2019irm} extend this idea to learning invariant representations for the multi-task supervised learning setting with a gradient penalty, and \citet{krueger2020rex} impose a similar constraint on invariance to risk across tasks to extrapolate to new, unseen tasks.
\citet{ilse2020designing} also draw parallels between domain generalization and causal inference, but only focus on data augmentation -- pointing out that certain types of data augmentation can harm downstream performance.
\citet{zhang2020invariant} also use an invariance objective in a multi-task setting for block MDPs, a family of tasks where only the rendering functions changes, e.g. camera angle or background. They assume these rendering functions are given and use a reconstruction-based loss to disentangle task-agnostic and task-specific representations. None of these methods utilize both rendering and post-rendering interventions or target the sim2real problem.

\section{DISCUSSION}
This paper presents Intervention-Based Invariant Transfer learning (\ouracronym): a method that encourages robustness to visual variations introduced at training time, which in turn, learns to attend to the true causal mechanisms of the underlying visual changes. Our experiments show that techniques such as Domain Randomization and Data Augmentation are forms of interventions --  namely rendering and post-rendering interventions -- that are effective for sim2real generalizability. In order to have guarantees on generalization capabilities, we propose to combine these interventions with an invariance objective that focuses on removing distractions in the background setting, enabling the agent to attend to the visual features relevant to solving the task at hand. However one difficulty we found with latent model-based methods, was the sensitivity of these methods to the choice of hyperparameters. We found that REx also suffers empirically from training volatility, likely caused by large penalties early on in training. 
Future work could explore avenues to improve stability in such invariance methods.
Working with the Metaworld tasks, we also found the multi-goal setting too challenging for an RL agent when dealing with rich observations. Further work could also explore approaches towards more robust multi-goal RL agents.


\section*{ACKNOWLEDGMENT}
We thank Johanna Hansen for the help with the real robot arm setup.
This work was supported by the Natural Sciences and Engineering Research Council (NSERC) of Canada.


\printbibliography

\pagebreak
\clearpage
\section*{APPENDIX}
\subsection{Implementation Details}
\paragraph{Learning a bisimulation metric}
Defining a distance $d$ between states requires defining both a distance between rewards (to soften \cref{eq:bisim-discrete-r}), 
and distance between state distributions (to soften \cref{eq:bisim-discrete-p}). Prior works use the Wasserstein metric for the latter, originally used in the context of bisimulation metrics by \citet{vanbreugel2001QuantitativeVerificationProbabilistic}.
The $p^{\text{th}}$ Wasserstein metric is defined between two probability distributions $\mathcal{P}_i$ and $\mathcal{P}_j$ as
\begin{equation}
W_p(\mathcal{P}_i,\mathcal{P}_j;d)=(\inf_{\gamma'\in\Gamma(\mathcal{P}_i,\mathcal{P}_j)}\int_{\mathcal{S}\times \mathcal{S}} d(s_i,s_j)^p\, \textnormal{d}\gamma'(s_i,s_j))^{1/p}
\end{equation}

where $\Gamma(\mathcal{P}_i,\mathcal{P}_j)$ is the set of all couplings of $\mathcal{P}_i$ and $\mathcal{P}_j$.
This is known as the ``earth mover'' distance, denoting the cost of transporting mass from one distribution to another~\cite{optimaltransport}.


\subsection{Hyperparameter Details}
The following parameters are the best performing parameters for \ouracronym. The training time and episode lengths depend on the environment and the task at hand. Table \ref{tab:hyperparam} shows the complete set of parameters and possible ranges for some parameters we experimented with, and the rest show specific parameters per task.

\begin{center}
\begin{table}[h]
\centering
\begin{tabular}{|l|l|}
  \hline
  Parameter name & Value \\ 
  \hline
  Replay buffer capacity & 100000 \\
  Batch size & [32-256] \\
  Image frame size & $ 84 \times 84 $ \\
  Image frame stack & 5 \\
  Discount $\gamma$ & 0.99 \\
  Optimizer & Adam \\
  Critic learning rate & $[1e^{-3} - 1e^{-5}]$  \\
  Critic target update frequency & [1-4] \\
  Critic Q-function soft-update rate $\tau_Q$ & 0.005 \\
  Critic encoder soft-update rate $\tau_Q$ & 0.005 \\
  Actor learning rate &  $[1e^{-3} - 1e^{-5}]$ \\
  Actor target update frequency & [2-4] \\
  Actor log stddev bounds & [$-5$,2] \\
  Encoder learning rate & $10^{-5}$ \\
  Temperature learning rate & $10^{-4}$ \\
  Temperature Adam’s $\beta_1$ & $0.9$ \\
  Init temperature & $0.1$ \\
  Penalty weight & $[0.1 - 1.0]$ \\
  Penalty anneal iterations & $[8000 - 20000]$ \\
  Training env batch size & 5 \\
  Training env resampling rate & $[150 - 3000]$ \\
  \hline
\end{tabular}
\centering
\caption{\label{tab:hyperparam} \small A complete overview of common hyper parameters for Sawyer and Jaco environments.}
\end{table}
\end{center}

\begin{center}
\begin{table}[h]
\centering
\begin{tabular}{|l|l|}
  \hline
  Parameter name & Value \\ 
  \hline
  Batch size & 256 \\
  Critic learning rate & $0.005$  \\
  Critic target update frequency & 2 \\
  Actor learning rate &  $0.005$ \\
  Actor target update frequency & 2 \\
  Encoder learning rate & $0.005$ \\
  Penalty weight & $0.5$ \\
  Penalty anneal iterations & $8000$ \\
  Training env batch size & 5 \\
  Training env resampling rate & $150$ \\
  \hline
\end{tabular}
\centering
\caption{\label{tab:hyperparam} \small Environment specific hyper parameters for Sawyer Reach environment.}
\end{table}
\end{center}

\begin{center}
\begin{table}[h]
\centering
\begin{tabular}{|l|l|}
  \hline
  Parameter name & Value \\ 
  \hline
  Batch size & 32 \\
  Critic learning rate & $0.001$  \\
  Critic target update frequency & 1 \\
  Actor learning rate &  $0.001$ \\
  Actor target update frequency & 4 \\
  Encoder learning rate & $0.001$ \\
  Penalty weight & $0.13$ \\
  Penalty anneal iterations & $6000$ \\
  Training env batch size & 5 \\
  Training env resampling rate & $300$ \\
  \hline
\end{tabular}
\centering
\caption{\label{tab:hyperparam} \small Environment specific hyper parameters for DMSuite Jaco Reach environment.}
\end{table}
\end{center}

\begin{center}
\begin{table}[h]
\centering
\begin{tabular}{|l|l|}
  \hline
  Parameter name & Value \\ 
  \hline
  Batch size & 64 \\
  Critic learning rate & $0.001$  \\
  Critic target update frequency & 2 \\
  Actor learning rate &  $0.001$ \\
  Actor target update frequency & 2 \\
  Encoder learning rate & $0.001$ \\
  Penalty weight & $0.5$ \\
  Penalty anneal iterations & $10000$ \\
  Training env batch size & 5 \\
  Training env resampling rate & $3000$ \\
  \hline
\end{tabular}
\centering
\caption{\label{tab:hyperparam} \small Environment specific hyper parameters for Sawyer Window Open environment.}
\end{table}
\end{center}

\begin{center}
\begin{table}[h]
\centering
\begin{tabular}{|l|l|}
  \hline
  Parameter name & Value \\ 
  \hline
  Batch size & 64 \\
  Critic learning rate & $0.002$  \\
  Critic target update frequency & 2 \\
  Actor learning rate &  $0.002$ \\
  Actor target update frequency & 2 \\
  Encoder learning rate & $0.002$ \\
  Penalty weight & $0.8$ \\
  Penalty anneal iterations & $10000$ \\
  Training env batch size & 5 \\
  Training env resampling rate & $3000$ \\
  \hline
\end{tabular}
\centering
\caption{\label{tab:hyperparam} \small Environment specific hyper parameters for Sawyer Window Close environment.}
\end{table}
\end{center}

\begin{center}
\begin{table}[h]
\centering
\begin{tabular}{|l|l|}
  \hline
  Parameter name & Value \\ 
  \hline
  Batch size & 32 \\
  Critic learning rate & $0.001$  \\
  Critic target update frequency & 2 \\
  Actor learning rate &  $0.001$ \\
  Actor target update frequency & 2 \\
  Encoder learning rate & $0.001$ \\
  Penalty weight & $0.2$ \\
  Penalty anneal iterations & $3000$ \\
  Training env batch size & 5 \\
  Training env resampling rate & $150$ \\
  \hline
\end{tabular}
\centering
\caption{\label{tab:hyperparam} \small Environment specific hyper parameters for Sawyer Button Press environment.}
\end{table}
\end{center}

\end{document}